%% file: Main.tex
\documentclass[letterpaper, 10pt, conference]{ieeeconf}
\IEEEoverridecommandlockouts
\usepackage{amsmath,amsfonts,amssymb}
\usepackage{algorithmic}
\usepackage{algorithm}
\usepackage{array}
\usepackage[caption=false,font=footnotesize,labelfont=rm,textfont=rm]{subfig}
\usepackage{subfig}
\usepackage{textcomp}
\usepackage{stfloats}
\usepackage{url}
\usepackage{verbatim}
\usepackage{graphicx}
\usepackage{cite}
\usepackage{mathrsfs}
\usepackage{booktabs}
\usepackage{siunitx}
\usepackage{balance}
\usepackage{leftindex}
\usepackage{pifont}
\usepackage{xcolor}
\usepackage{diagbox}

\usepackage{listings}

\usepackage{ntheorem}
\theoremstyle{plain}
\newtheorem{definition}{Definition}

\newtheorem{lemma}{Lemma}
\newtheorem{theorem}{Theorem}
\newtheorem{remark}{Remark}
\newtheorem{corollary}{Corollary}

\title{\LARGE \bf Static Timing Orchestration for Tree-Structured \\ Robot Control Firmware}

\author{Wang Xi, Feiran Wei, Mo Deng, Weiheng Lin, Pangkit Fong and Jianping He
\thanks{The authors are with the Department of Automation, Shanghai Jiao Tong University, and Key Laboratory of System Control and Information Processing, Ministry of Education of China, Shanghai, China. Corresponding E-mail address: \{bddwyx,fpjgaoge,jphe\}@sjtu.edu.cn}
}

\begin{document}

\maketitle
\thispagestyle{empty}
\pagestyle{empty}

\definecolor{Update}{rgb}{0.7176, 0.6627, 0.9215}
\definecolor{Handle}{rgb}{0.9333, 0.5059, 0.7412}

\newcommand{\SymbUpdate}{{\color{Update} $\blacksquare$} }
\newcommand{\SymbHandle}{{\color{Handle} \ding{108}} }
\newcommand{\FineMote}{\textsc{FineMote} }
\newcommand{\ROV}{\textsc{FINS-ROV} }

\begin{abstract}
As robotic systems become increasingly complex, generating control firmware from structural description files has emerged as a promising paradigm for reducing development complexity and improving maintainability.
Existing robot description formats naturally represent robotic systems as hierarchical tree structures, where devices are recursively composed into functional subsystems and eventually into the complete robot.
However, such tree-structured organization also introduces structured data dependencies that affect perception-to-decision latency and, consequently, control performance.

In this paper, we propose \FineMote, a control firmware generation framework with a scheduling mechanism tailored for tree-structured device models.
The framework objectifies heterogeneous low-level control logic and exposes unified scheduling units and execution entry points.
Based on the resulting object hierarchy, the scheduling mechanism exploits compile-time information to statically determine execution order with minimal runtime overhead.
We prove that the proposed mechanism satisfies deadline and precedence constraints, and further derive an upper bound on intra-tree decision latency.
We implement the proposed framework and evaluate it on real robotic control platforms.
The experimental results show improved timing behavior and runtime responsiveness, demonstrating the practical effectiveness of the proposed design.
\end{abstract}

\input{Contents/introduction}
\input{Contents/Formulation}
\input{Contents/Scheduling}
\input{Contents/Latency}
\input{Contents/Experiment}
\input{Contents/Related_Works}
\input{Contents/Conclusion}
\input{Contents/Appendix}
\balance

\bibliography{Refs} 
\bibliographystyle{IEEEtran}

\vfill

\end{document}

%% file: Contents/Introduction.tex
\section{Introduction}

With the increasing scale of robotic applications, scheduling strategies for robotic systems have attracted growing attention.
To interact reliably with the physical world, robots are typically equipped with low-level control firmware that drives sensors, actuators, and other physical devices.
Modern robots integrate many heterogeneous devices through multiple communication buses, making such firmware highly complex.
Consequently, developing and maintaining control firmware through conventional manual workflows is difficult and error-prone.

A common way to manage this complexity is to construct firmware from structural descriptions.
Robot description formats such as the Unified Robot Description Format (URDF)~\cite{URDF} and xacro~\cite{xacro} were introduced for robot modeling and analysis, and their XML-like structure naturally describes a tree-structured device hierarchy.
For example, leaf devices such as motors form higher-level components such as manipulators, which are further composed into a complete robot system.
Such structural information can serve as an input to firmware generation frameworks.
For instance, ros2\_control~\cite{ros2_control_docs} uses dedicated URDF fields to specify hardware interfaces and binds them with device libraries to construct control code.
Similarly, AUTOSAR~\cite{AUTOSAR} adopts an XML-based paradigm for control software generation in the automotive domain.

However, building such a framework is itself nontrivial.
A generation framework must bridge high-level structural descriptions and low-level firmware implementation, objectify heterogeneous device logic, bind communication interfaces, and expose schedulable execution units.
Moreover, even when the functional structure can be generated, timing orchestration remains insufficiently addressed.
A tree-structured device hierarchy induces data dependencies among tasks: state information is aggregated from leaves to the root, while decisions are propagated from the root back to leaves.
Improper scheduling may therefore accumulate latency along this bidirectional propagation process, leading to longer perception-to-decision delay and degraded control performance.

Reducing this latency is difficult because timing orchestration must satisfy both precedence and deadline constraints.
To expose useful scheduling points, control firmware often decomposes each device task into two stages: one stage provides state information for upper-level decision generation, and the other generates local actions according to upper-level commands.
These two stages must execute sequentially within the same device period, introducing intra-task precedence constraints.
As analyzed later, such constraints make even single-core scheduling nontrivial.
Furthermore, execution times of user-defined device logic are usually unavailable at the framework level, so aggressive task switching or preemption may lead to deadline violations.

We observe that robot control firmware has a largely static structure.
Robot components do not appear or disappear during firmware execution, and their mechanical structure and hardware configuration are available at compile time.
This motivates a static timing orchestration strategy: scheduling policies can be constructed before runtime by leveraging framework-level design mechanisms and compile-time structural information.
Such a strategy matches embedded control platforms, where low runtime overhead and deterministic behavior are essential.

Based on these observations, we design a control firmware generation framework with static timing orchestration.
The framework targets complex robotic mechatronic systems and aims to simplify firmware development while providing predictable timing behavior.
The main contributions of this paper are summarized as follows:

\begin{itemize}
    \item We propose an object-oriented control firmware generation framework for tree-structured robotic devices.
    Based on this framework, we formulate the corresponding scheduling problem and explicitly capture the impact of tree-structured device dependencies on intra-tree decision latency.

    \item We design a static scheduling strategy that satisfies both schedulability and intra-device precedence constraints.
    We further derive an upper bound on the intra-tree decision latency under the proposed scheduling strategy.

    \item We provide an open-source implementation of the proposed framework and evaluate its feasibility and scheduling behavior on a real robotic platform.
    The results demonstrate that the framework can support practical firmware development while maintaining predictable timing behavior.
\end{itemize}

%% file: Contents/Formulation.tex
\section{Framework Design}

This section presents the framework design from the perspective of scheduling and formulates the corresponding scheduling problem.
This abstraction allows readers to follow the technical development without requiring detailed understanding of the underlying framework implementation.

To avoid tedious implementation details and conserve space, implementation mechanisms are described using \emph{design pattern} names.
Comprehensive definitions of design patterns can be found in~\cite{gamma1995design, mcconnell2004code}.

We first provide an overview of the framework.
The proposed firmware generation framework adopts a device-centric scheduling model.
Each device object corresponds to a physical robot component and aligns with the structural description file.
Developers define device-specific functionality following the framework interface, while scheduling and timing behaviors are enforced at the framework level without user intervention.


\subsection{Device-Centric Task Model}

The first key idea is to decompose complex robotic workloads into schedulable units.
In conventional robotic control firmware, two types of execution activities usually coexist:
time-triggered control computation and event-triggered communication processing through physical peripherals.
However, scheduling such hybrid-triggered workloads is difficult, since periodic control tasks and peripheral events follow different activation mechanisms.
Therefore, we first simplify the scheduling model by separating communication processing from control computation.

This separation is based on the following observation.
Within one control period, each device object exchanges only a bounded amount of data through physical interfaces, and the cost of moving such data is relatively small.
Thus, we assign higher priority to communication peripherals, while restricting their responsibility to data movement.
In particular, communication paths do not perform protocol parsing or execute control logic.
This design keeps peripheral responses timely and, meanwhile, avoids blocking periodic control computation with complex event-triggered processing.

Furthermore, this separation can be directly realized in the framework implementation.
Each device object registers its communication requirements to the corresponding bus object through the \emph{Observer} pattern.
Meanwhile, each bus object is implemented as a \emph{Singleton}, which handles transmission requests and dispatches received data through interrupts or DMA.
As a result, event-triggered peripheral processing is confined to asynchronous data transfer, whereas device-level control logic is executed in the time-triggered part of the firmware.

Under this implementation, the scheduling problem of the control firmware is largely reduced to time-triggered scheduling.
Therefore, the remaining scheduling analysis can focus on periodic device tasks.
Let \(\delta\) denote the system scheduling tick.
Consider a robotic application with \(N\) device objects \(d_n\), \(n=1,\ldots,N\).
Let $\mathbb{D}=\{d_1,\ldots,d_N\}$ denote the set of device objects instantiated in the robotic application.
Then we define the basic scheduling abstraction used in this paper.

\begin{definition}[Device Task Model]
Each device \(d_n \in \mathbb{D}\) periodically releases a task \(\tau_n\), characterized by
\[
    (\mathcal{C}_n,\mathcal{T}_n,\mathcal{D}_n),
\]
where \(\mathcal{C}_n\) is the worst-case execution time (WCET),
\(\mathcal{T}_n\) is the period,
and \(\mathcal{D}_n\) is the relative deadline.
\(\mathcal{T}_n\) is assumed to be an integer multiple of the scheduling tick \(\delta\).
\end{definition}

We assume that scheduling starts at \(t=0\), when all device objects release their first tasks simultaneously.
Considering the characteristics of embedded control firmware platforms, the remainder of this paper focuses on preemptive single-core scheduling.
Furthermore, to simplify the analysis, we assume implicit deadlines for device tasks:
\[
    \mathcal{T}_n = \mathcal{D}_n,
    \quad \forall n.
\]

At the implementation level, all device objects are additionally required to be globally instantiated.
This requirement follows from the robotic control setting, where the set of physical devices is fixed at system startup.
Moreover, global instantiation enables deterministic startup orchestration, which will be used later to reason about device ordering.


\subsection{Tree-Structured Dependency}

The device task model defined above abstracts each device object as a periodic scheduling unit.
However, these units are not independent.
In structure-driven robot firmware, device objects are organized according to the robot hierarchy described by the structural specification.
Therefore, the next step is to model the dependency relation induced by this hierarchy.

\begin{figure}[h]
\centering
\includegraphics[width=\linewidth]{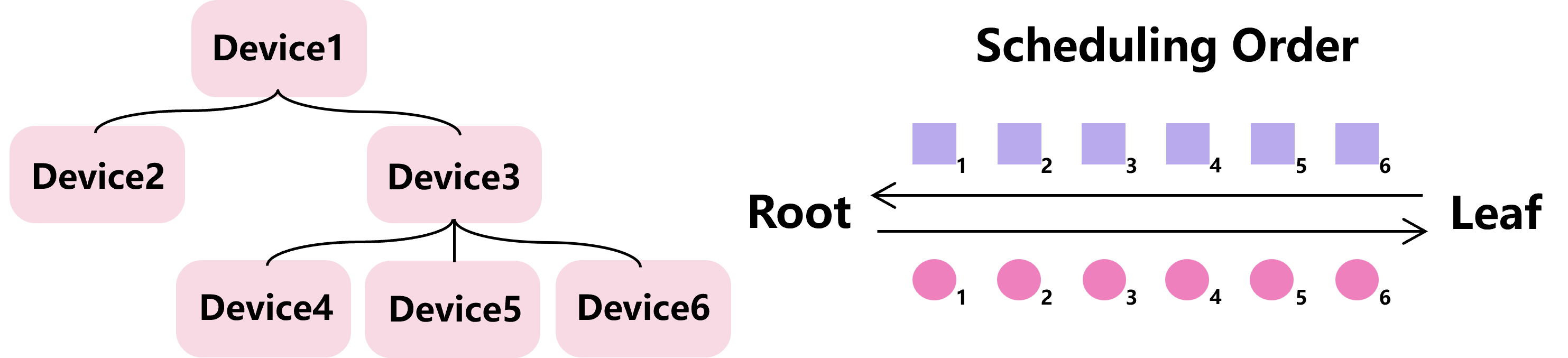}
\caption{Tree-structured device dependencies and an example low-latency scheduling order.}
\label{fig: Device}
\end{figure}

Although structural descriptions provide a natural basis for firmware generation, directly generating complete device logic from them is often impractical.
The main difficulty comes from the diversity of leaf devices.
Consider a manipulator as an example.
The internal motion-control logic of the manipulator is largely independent of the specific motor model, since the controller only needs to issue commands such as position, velocity, or torque to motor objects.
However, replacing the motor hardware may still change implementation-level details, including data types, variables, and configuration parameters.
This mismatch between functional abstraction and hardware-specific implementation prevents structural descriptions from directly specifying complete device functions.

To address this issue, the framework represents hierarchical composition through \emph{dependency injection}.
Leaf device objects are first declared or instantiated, and are then passed as construction arguments to upper-level device objects.
Meanwhile, upper-level classes access lower-level devices through abstract interfaces, which allows the same composite logic to be reused across different hardware implementations.
Consequently, defining device objects from leaves to roots reconstructs the robot functionality while preserving hardware abstraction.

Based on this implementation mechanism, we define the tree-structured dependency among device objects.

\begin{definition}[Tree-structured dependency]
For two devices \(d_{n_1}, d_{n_2}\in\mathbb{D}\), the ordered pair
\[
    (d_{n_1}, d_{n_2})
\]
denotes that \(d_{n_2}\) depends on \(d_{n_1}\) in the structural description.
All such dependency pairs form the edge set \(\mathbb{E}\).
\end{definition}

Due to the structural constraints of the description files, the directed graph \((\mathbb{D},\mathbb{E})\) is assumed to be a forest.
Let \(M\) denote the number of trees.
We denote the \(m\)-th tree by \((\mathbb{D}_m,\mathbb{E}_m)\), where \(m=1,\ldots,M\).

For subsequent latency analysis, we also define leaf-to-root paths in each tree.
Consider a sequence of device indices \((n_p)_{p=1}^{P}\) in the \(m\)-th tree such that
\[
    d_{n_p}\in\mathbb{D}_m,\quad \forall p=1,\ldots,P,
\]
and
\[
    (d_{n_p},d_{n_{p+1}})\in\mathbb{E}_m,\quad \forall p=1,\ldots,P-1.
\]
If \(d_{n_1}\) is a leaf node and \(d_{n_P}\) is the root node of the \(m\)-th tree, then \((n_p)_{p=1}^{P}\) is called a \emph{leaf-to-root path}.
The set of all leaf-to-root paths in the \(m\)-th tree is denoted by \(\mathbb{P}_m\).
This notation will be used later to define decision latency.


\subsection{Intra-Tree Latency Model}

\begin{figure*}[ht]
    \centering
    \subfloat[UML Graph]{
        \includegraphics[height=0.13\textheight,keepaspectratio]{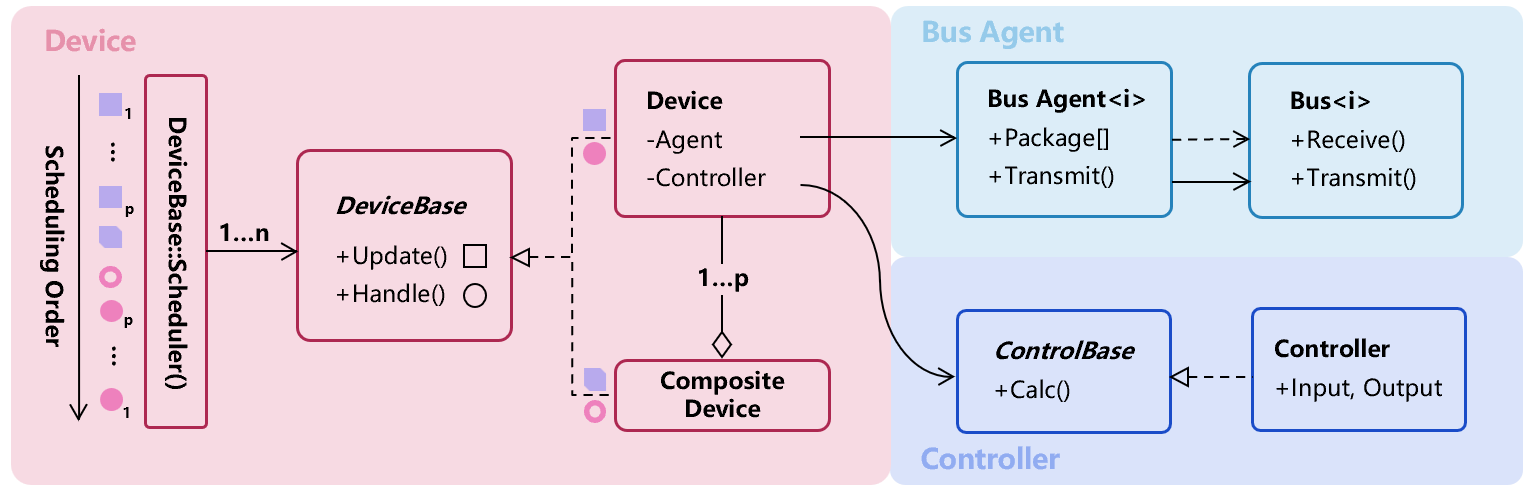}
        \label{fig:UML}
    }
    \hfill
    \subfloat[Structural Diagram]{
        \includegraphics[height=0.13\textheight,keepaspectratio]{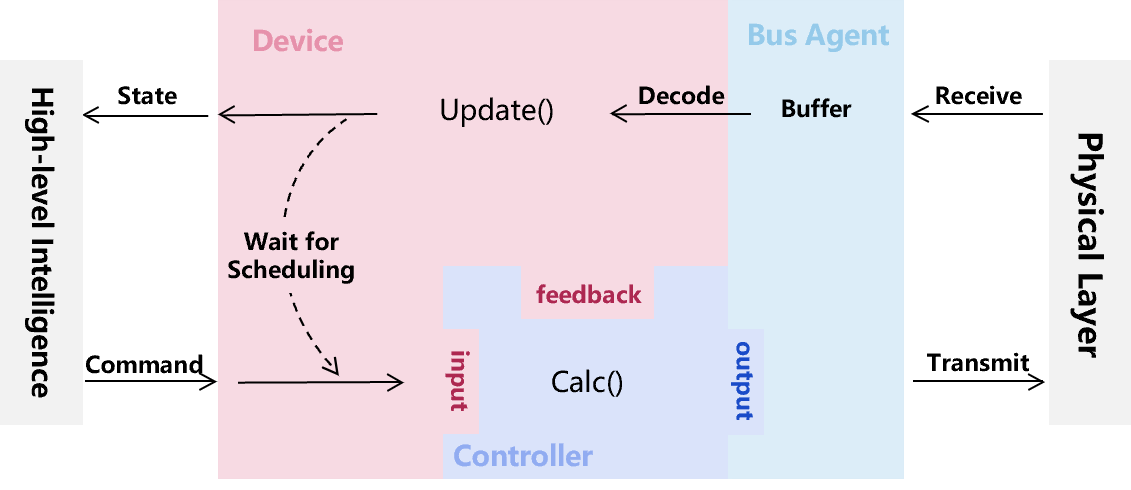}
        \label{fig:overview}
    }
    \caption{Core mechanisms of the framework.
        These mechanisms periodically interact with the physical layer to establish basic control.
        Due to space and scope limitations, this paper focuses on the modeling and scheduling of device objects.}
    \label{fig:core_mechanism}
\end{figure*}

The tree-structured dependency defined above introduces nontrivial timing interactions among device tasks.
In the upward direction, an upper-level device depends on state information produced by lower-level devices.
In the downward direction, decisions generated by upper-level devices are propagated back toward leaf devices and are eventually delivered to cyber-physical interfaces.
Therefore, even if each device task meets its own deadline, scheduling gaps between dependent devices may accumulate along a tree and increase the overall decision delay.

In this paper, we focus on the intra-tree decision latency, defined as the elapsed time from a leaf-side state update to the generation of the corresponding leaf-side actuator command.
This latency includes both computation time and scheduling-induced waiting time.
Moreover, because data must traverse multiple hierarchy levels, scheduling-induced waiting time may accumulate along a leaf-to-root-to-leaf propagation path.
As a result, the execution order of device tasks becomes critical to the timing behavior of the control firmware.

To expose schedulable points inside each device task, the framework decomposes the task of every device into two stages.
The \texttt{Update} stage, denoted by \(\tau_n^{+}\) and \SymbUpdate, updates the internal state of device \(d_n\) using data provided by lower-level devices.
It terminates when the state information required by upper-level devices becomes available.
The \texttt{Handle} stage, denoted by \(\tau_n^{-}\) and \SymbHandle, generates actions for the current cycle based on the updated device state and propagates decision information to lower-level devices.

This two-stage decomposition makes the bidirectional data propagation explicit.
After a lower-level \SymbUpdate stage finishes, the scheduler may execute the \SymbUpdate stage of an upper-level device without waiting for the lower-level \SymbHandle stage.
Similarly, after an upper-level decision is generated, the scheduler may propagate it downward through subsequent \SymbHandle stages.
Thus, the decomposition provides the scheduling flexibility needed to reduce latency accumulation inside a device tree.

Based on this decomposition, the set of schedulable stages is defined as
\[
    \mathbb{T}(\mathbb{D})
    =
    \{
        \tau_n^{+},
        \tau_n^{-}
        \mid
        d_n\in\mathbb{D}
    \}.
\]

To reason about instance-level data propagation, we next introduce release indices and timing functions.
For each device \(d_n\), let
\[
    \mathbb{K}_n
    \triangleq
    \{
        k \in \mathbb{N}
        \mid
        \exists \alpha\in\mathbb{N},\;
        k\delta=\alpha\mathcal{T}_n
    \}
\]
be the set of valid release indices; namely, \(k\in\mathbb{K}_n\) iff \(d_n\) releases a task at time \(k\delta\).

We model the scheduler by two timing functions
\(S,F:\mathbb{T}(\mathbb{D})\times\mathbb{N}\rightarrow\mathbb{R}\).
For a stage type \(\circ\in\{+,-\}\), \(S(\tau_n^{\circ},k)\) and \(F(\tau_n^{\circ},k)\) denote the start time and finish time of the stage instance \((\tau_n^{\circ},k)\), respectively.
Here, \((\tau_n^{\circ},k)\) denotes the instance of stage \(\tau_n^{\circ}\) released by device \(d_n\) at time \(k\delta\).

With the start-time and finish-time mappings defined, we now specify when data produced by one stage becomes visible to another stage.
Combining the tree dependency with scheduling behavior, we define the data transfer relation \(\prec\).

\begin{definition}[Data transfer relation]
Let \(\circ\in\{+,-\}\) denote the stage type.
For two stage instances \((\tau_{n_1}^{\circ},k_1)\) and \((\tau_{n_2}^{\circ},k_2)\), we write
\[
    (\tau_{n_1}^{\circ},k_1)
    \prec
    (\tau_{n_2}^{\circ},k_2)
\]
if the following conditions hold:
\begin{enumerate}
    \item The two devices are adjacent in the corresponding propagation direction:
    \[
        \begin{cases}
            (d_{n_1}, d_{n_2}) \in \mathbb{E}, & \circ = +, \\
            (d_{n_2}, d_{n_1}) \in \mathbb{E}, & \circ = -.
        \end{cases}
    \]

    \item The instance \((\tau_{n_1}^{\circ},k_1)\) is the latest released producer instance whose output is visible to \((\tau_{n_2}^{\circ},k_2)\), i.e.,
    \[
        k_1
        =
        \max
        \left\{
            k\in\mathbb{K}_{n_1}
            \;\middle|\;
            F(\tau_{n_1}^{\circ},k)
            \le
            S(\tau_{n_2}^{\circ},k_2)
        \right\}.
    \]
\end{enumerate}
\end{definition}

The relation \(\prec\) captures the data instance actually consumed by a stage instance.
For the \SymbUpdate stage, data propagates from a child device to its parent.
For the \SymbHandle stage, decision information propagates from a parent device to its child.

With this relation, we can define the latency along a leaf-to-root path.
Consider the \(i\)-th leaf-to-root path in the \(m\)-th tree, denoted by
\[
    (n_p)_{p=1}^{P_i} \in \mathbb{P}_m,
\]
where \(P_i\) is the length of this path, \(d_{n_1}\) is the leaf, and \(d_{n_{P_i}}\) is the root.

Given a terminal release index
\[
    k_{(2P_i-1)}\in\mathbb{K}_{n_1},
\]
corresponding to the terminal leaf-side \SymbHandle stage, the data-transfer relation determines a backward release-index sequence
\[
    (k_j)_{j=1}^{2P_i-1}.
\]
Specifically, the downward decision propagation is first determined backward by
\[
    (\tau_{n_{P_i}}^{-},k_{P_i})
    \prec
    (\tau_{n_{(P_i-1)}}^{-},k_{(P_i+1)})
    \prec
    \cdots
    \prec
    (\tau_{n_{1}}^{-},k_{(2P_i-1)}),
\]
where \(n_{(P_i+v)}=n_{(P_i-v)}\).
Then, the upward state propagation is determined backward by
\[
    (\tau_{n_1}^{+},k_1)
    \prec
    (\tau_{n_2}^{+},k_2)
    \prec
    \cdots
    \prec
    (\tau_{n_{P_i}}^{+},k_{P_i}).
\]

The decision latency of this path ending at release index \(k_{(2P_i-1)}\) is defined as
\[
    F(\tau_{n_1}^{-},k_{(2P_i-1)})
    -
    S(\tau_{n_1}^{+},k_1).
\]
That is, latency is measured from the start of the leaf-side \SymbUpdate stage that produced the state eventually used by the terminal leaf-side \SymbHandle stage, to the completion of that terminal \SymbHandle stage.

We now lift this path-level metric to the tree level.

\begin{definition}[Intra-tree decision latency]
    \label{def:intra_tree_latency}
    For the \(m\)-th tree, the intra-tree decision latency is defined as
    \[
        L_m
        =
        \max_{i:\,(n_p)_{p=1}^{P_i}\in\mathbb{P}_m}
        \;
        \max_{k_{(2P_i-1)}}
        \left(
            F(\tau_{n_1}^{-},k_{2P_i-1})
            -
            S(\tau_{n_1}^{+},k_1)
        \right),
    \]
    where the release indices \(k_1,\ldots,k_{(2P_i-2)}\) are determined by the data transfer relation \(\prec\).
\end{definition}


\subsection{The Scheduling Problem}

Given a scheduling policy \(S\), the finish-time function \(F\) is induced by the resulting execution trace.
Therefore, designing the scheduling policy \(S\) is sufficient to determine both the start-time and finish-time behavior of all stage instances.

We now formulate the scheduling problem for tree-structured device dependencies.
Given the timing functions defined above, the objective is to find a scheduling policy \(S\) that minimizes the worst intra-tree decision latency:
\begin{subequations} \label{eq:original_problem}
    \begin{align}
        & \min_{S} \; \max_{m=1,\ldots,M} \; L_m \label{subeq:latency_cost} \\
    \text{s.t.}\quad
        & F(\tau_n^{-}, k) \le k\delta + \mathcal{D}_n,
        \quad \forall n,\; \forall k \in \mathbb{K}_n, \label{subeq:schedulability} \\
        & F(\tau_n^{+}, k) \le S(\tau_n^{-}, k),
        \quad \forall n,\; \forall k \in \mathbb{K}_n. \label{subeq:precedence}
    \end{align}
\end{subequations}
Here, \eqref{subeq:schedulability} requires each device task to finish before its relative deadline, while \eqref{subeq:precedence} enforces the intra-device precedence from the \SymbUpdate stage to the \SymbHandle stage.

Solving~\eqref{eq:original_problem} is challenging for several reasons.
First, the scheduler has limited information at the framework level.
The execution time \(\mathcal{C}_n\) is generally difficult to obtain, and the framework has limited prior knowledge of the user-defined device set \(\mathbb{D}\) before compilation.
Consequently, using fine-grained priority assignment to resolve intra-tree data dependencies is difficult and may compromise schedulability.

Second, the optimization problem is computationally difficult.
Without preemption, the problem is consistent with the ST-SR-IA problem with in-schedule dependencies defined in~\cite{korsah2013comprehensive}.
The precedence constraint in~\eqref{subeq:precedence} turns the assignment problem from a tractable combinatorial optimization problem into an NP-hard one.
Although we do not claim a direct complexity proof for the preemptive variant considered in this paper, it is reasonable to regard it as computationally hard as well, given its close relation to the non-preemptive case and the additional scheduling choices introduced by preemption.
Therefore, even verifying whether a given scheduling policy is optimal can be difficult.

For these reasons, we do not seek an exact optimum.
Instead, we design a simple and practically implementable scheduling strategy that satisfies the schedulability and precedence constraints while reducing intra-tree decision latency.

%% file: Contents/Scheduling.tex
\section{A Static Scheduling Strategy}

\subsection{Device Registration}

Static scheduling requires an ordering of device objects that is consistent with both the device hierarchy and data-propagation directions. However, the firmware does not directly observe the complete device forest; it only observes construction and registration events. Maintaining an explicit dependency map is possible but fragile, especially when device objects are distributed across multiple translation units, while dynamically storing the full graph would conflict with the deterministic and low-overhead requirements of embedded firmware.

We therefore exploit the temporal implication of \emph{dependency injection}. If a parent device is constructed from references to its child devices, then the child devices must already be available before the parent is constructed. This construction order induces a partial dependency order that can be captured by registration. The remaining issue is translation-unit-level initialization: non-local objects are constructed before entering \texttt{main}, but their relative order across different translation units is generally not controlled, leading to the well-known Static Initialization Order Fiasco~\cite{SIOF}. We first formalize how device objects appear at the translation-unit level.

\begin{definition}[Devices in Translation Units]
The declarations and definitions of device objects in \(\mathbb{D}\) are distributed over \(Q\) translation units.
For each translation unit \(q\), let \(\mathbb{D}^{(q)}\subseteq\mathbb{D}\) denote the set of device objects declared or defined in it.
We denote their source-level appearance order by
\[
    \mathbf{U}_q
    \triangleq
    \left(d^{(q)}_r\right)_{r=1}^{|\mathbb{D}^{(q)}|},
\]
where \(d^{(q)}_r\) is the \(r\)-th device object declared or defined in translation unit \(q\).
\end{definition} 

According to \S~6.6.3 of C++17~\cite{cpp17} and mainstream compiler implementations, non-local objects are constructed before entering the main function.
We denote the resulting device initialization order by
\[
    \mathbf{O}
    \triangleq
    \left(d_{n_s}\right)_{s=1}^{N},
\]
where \((n_s)_{s=1}^{N}\) is a permutation of \(\{1,\ldots,N\}\), and \(s\) denotes the initialization rank.
Accordingly, we define the initialization-rank mapping
\[
    \mathcal{I}(d_{n_s}) = s .
\]

We now state the ordering property used by the proposed registration mechanism.
\begin{theorem}[Initialization Order Preservation \label{thm: initialization-order-preservation}]
Assume that the program is well-formed under C++17~\cite{cpp17} and follows the \emph{dependency-injection} discipline described above.
Then the following properties hold:
\begin{enumerate}
    \item For every translation unit \(q\), the source-level device sequence \(\mathbf{U}_q\) is a subsequence of the global initialization order \(\mathbf{O}\).
    \item The initialization order \(\mathbf{O}\) is consistent with the device forest. That is, for any two devices \(d_{n_1},d_{n_2}\in\mathbb{D}\), if
    \[
        (d_{n_1},d_{n_2})\in\mathbb{E},
    \]
    then
    \[
        \mathcal{I}(d_{n_1}) < \mathcal{I}(d_{n_2}).
    \]
\end{enumerate}
\end{theorem}

\begin{proof}
This result strongly depends on the framework mechanism and the partially-ordered initialization feature introduced in C++17.
Appendix~\ref{appendix: initialization-order} provides a feasible implementation path that satisfies this requirement.
\end{proof}

Theorem~\ref{thm: initialization-order-preservation} provides a useful perspective.
The framework does not need to explicitly reconstruct the complete dependency forest or require developers to manually specify a traversal order.
Instead, under the proposed construction discipline, the initialization order already provides a deterministic child-to-parent order.
We next use this order to design the static scheduling algorithm.


\subsection{Static Device Scheduling}

The registration order provides a deterministic order for device objects.
Before scheduling starts, the framework constructs all scheduling sequences statically from this order, without maintaining the device forest at runtime.

Devices are first partitioned by their periods.
Let \(W\) be the number of distinct device periods \(\mathcal{T}\).
For each \(w=1,\ldots,W\), let \(\mathbf{Q}_w\) denote the subsequence of the initialization order \(\mathbf{O}\) containing all devices with the same period \(\mathcal{T}^{(w)}\):
\[
    \mathbf{Q}_w
    \triangleq
    \left(
        d_{n_s}\in\mathbf{O}
        \mid
        \mathcal{T}_{n_s}=\mathcal{T}^{(w)}
    \right),
\]
where \(\mathcal{T}^{(w)}\) is the common period of devices in \(\mathbf{Q}_w\).
The periods are ordered as
\[
    \mathcal{T}^{(1)}
    <
    \mathcal{T}^{(2)}
    <
    \cdots
    <
    \mathcal{T}^{(W)} .
\]

The framework assigns static priorities to these sequences according to Rate-Monotonic Scheduling (RMS).
Thus, a sequence with a shorter period has a higher priority, and a higher-priority sequence may preempt the execution of a lower-priority sequence.

Within each sequence, the execution order is fixed.
For
\[
    \mathbf{Q}_w
    =
    (d_{n_1},d_{n_2},\ldots,d_{n_{|\mathbf{Q}_w|}}),
\]
the scheduler executes
\begin{equation} \label{eq: scheduling order}
    \tau_{n_1}^{+},
    \tau_{n_2}^{+},
    \ldots,
    \tau_{n_{|\mathbf{Q}_w|}}^{+},
    \tau_{n_{|\mathbf{Q}_w|}}^{-},
    \ldots,
    \tau_{n_2}^{-},
    \tau_{n_1}^{-}.
\end{equation}

Since each period bucket \(\mathbf{Q}_w\) preserves the initialization order, same-period dependencies follow the desired child-to-parent order for \SymbUpdate and parent-to-child order for \SymbHandle.

Overall, the scheduling policy is fully determined before runtime: period buckets, their priorities, and their internal execution orders are all fixed before scheduling begins.


\subsection{Schedulability Analysis}

We now analyze the schedulability of the proposed static scheduling policy.
The analysis first abstracts event-triggered bus responses as conservative periodic workloads, and then combines them with the period buckets defined above.

Let \(N_{\mathcal{B}}\) denote the number of bus objects used in the firmware.
For each bus object \(b_\beta\), where \(\beta=1,\ldots,N_{\mathcal{B}}\), let \(\mathscr{R}_\beta\) be its baud rate and \(\ell_\beta\) be the minimum frame length in bits.
The minimum response interval is
\[
    \mathcal{T}^{\mathcal{B}}_\beta
    =
    \frac{\ell_\beta}{\mathscr{R}_\beta}.
\]
The corresponding service routine performs only data movement, with WCET \(\mathcal{C}^{\mathcal{B}}_\beta\).

\begin{lemma}[Periodic Abstraction of Bus Responses]
\label{lemma:bus_periodic_abstraction}
For schedulability analysis, each bus object \(b_\beta\) can be conservatively modeled as a periodic workload
\[
    \left(
        \mathcal{C}^{\mathcal{B}}_\beta,
        \mathcal{P}^{\mathcal{B}}_\beta,
        \mathcal{T}^{\mathcal{B}}_\beta
    \right),
\]
where
\[
    \mathcal{P}^{\mathcal{B}}_\beta
    =
    \min\{\mathcal{T}^{\mathcal{B}}_\beta,\delta\},
\]
and \(\delta\) is the system scheduling tick.
\end{lemma}

\begin{proof}
Bus responses are separated by at least \(\mathcal{T}^{\mathcal{B}}_\beta\), and therefore form a sporadic workload.
Replacing this workload with a periodic workload of period \(\mathcal{P}^{\mathcal{B}}_\beta=\min\{\mathcal{T}^{\mathcal{B}}_\beta,\delta\}\) is conservative since \(\mathcal{P}^{\mathcal{B}}_\beta\le\mathcal{T}^{\mathcal{B}}_\beta\).
For any interval of length \(L\), the processor demand of the original workload satisfies
\[
    \left\lceil
        \frac{L}{\mathcal{T}^{\mathcal{B}}_\beta}
    \right\rceil
    \mathcal{C}^{\mathcal{B}}_\beta
    \le
    \left\lceil
        \frac{L}{\mathcal{P}^{\mathcal{B}}_\beta}
    \right\rceil
    \mathcal{C}^{\mathcal{B}}_\beta .
\]
Thus, the periodic abstraction safely upper-bounds the processor demand of bus responses.
\end{proof}

Next, we abstract each period bucket as one periodic workload.
For a bucket \(\mathbf{Q}_w\), let
\[
    \mathcal{C}^{\mathbf{Q}}_w
    \triangleq
    \sum_{d_n\in\mathbf{Q}_w}\mathcal{C}_n
\]
denote its aggregate execution time, where \(\mathcal{C}_n\) includes both the \SymbUpdate and \SymbHandle stages of device \(d_n\).
Since all devices in \(\mathbf{Q}_w\) share the same period \(\mathcal{T}^{(w)}\), the bucket can be modeled as a periodic workload
\[
    \tau^{\mathbf{Q}}_w
    =
    \left(
        \mathcal{C}^{\mathbf{Q}}_w,
        \mathcal{T}^{(w)},
        \mathcal{T}^{(w)}
    \right).
\]

Since every device period is an integer multiple of the scheduling tick \(\delta\), and \(\mathcal{P}^{\mathcal{B}}_\beta \le \delta\), bus workloads have priorities no lower than device buckets under Rate-Monotonic Scheduling (RMS).
This is consistent with the implementation, where bus service routines are assigned higher priority than periodic control execution.

\begin{theorem}[Schedulability Bound]
\label{thm:schedulability_bound}
Consider \(N_{\mathcal{B}}\) bus objects and \(W\) period buckets scheduled according to the proposed static policy.
A sufficient condition for schedulability is
\[
    \sum_{\beta=1}^{N_{\mathcal{B}}}
    \frac{\mathcal{C}^{\mathcal{B}}_\beta}{\mathcal{P}^{\mathcal{B}}_\beta}
    +
    \sum_{w=1}^{W}
    \frac{\mathcal{C}^{\mathbf{Q}}_w}{\mathcal{T}^{(w)}}
    \le
    (N_{\mathcal{B}}+W)
    \left(
        2^{1/(N_{\mathcal{B}}+W)}-1
    \right).
\]
\end{theorem}

\begin{proof}
By Lemma~\ref{lemma:bus_periodic_abstraction}, each bus object \(b_\beta\) is conservatively modeled as a periodic workload with execution time \(\mathcal{C}^{\mathcal{B}}_\beta\) and period \(\mathcal{P}^{\mathcal{B}}_\beta\).
Meanwhile, each period bucket \(\mathbf{Q}_w\) is modeled as a periodic workload with execution time \(\mathcal{C}^{\mathbf{Q}}_w\), period \(\mathcal{T}^{(w)}\), and implicit deadline.
Thus, the firmware is abstracted as \(N_{\mathcal{B}}+W\) periodic workloads scheduled under RMS.
Applying the Liu--Layland utilization bound~\cite{liu1973scheduling} yields the stated sufficient condition.
\end{proof}


\subsection{Summary of the Static Policy}

Using the ordering property in Theorem~\ref{thm: initialization-order-preservation}, we design a fully static scheduling policy.
The period buckets \(\mathbf{Q}_w\), their RMS priorities, and their internal execution orders are all determined before scheduling starts.
The execution order in~\eqref{eq: scheduling order} enforces the intra-device precedence constraint in~\eqref{subeq:precedence}, since each \SymbHandle stage is executed only after the corresponding \SymbUpdate stage.
Meanwhile, Theorem~\ref{thm:schedulability_bound} provides a sufficient condition under which the deadline constraint in~\eqref{subeq:schedulability} is satisfied.

Therefore, the remaining question is the latency cost of this static policy.
In the next section, we derive an upper bound on the intra-tree decision latency \(L_m\) of the proposed scheduling strategy.

%% file: Contents/Latency.tex
\section{Latency Analysis}


\subsection{Response-Time Bounds}

Before deriving latency bounds, we first analyze the response-time bounds of \SymbUpdate and \SymbHandle.
For device \(d_n\) and phase \(\tau_n^\circ\), where \(\circ\in\{+,-\}\), define its response time at release index \(k\in\mathbb{K}_n\) as
\[
    \mathcal{R}_n^\circ(k)
    \triangleq
    F(\tau_n^\circ,k)-k\delta .
\]
This response time depends on the specific release index \(k\), and is therefore inconvenient for deriving release-independent latency bounds.
Instead, we seek two uniform upper bounds \(\widehat{\mathcal{R}}_n^+\) and \(\widehat{\mathcal{R}}_n^-\) such that
\[
    \mathcal{R}_n^+(k)\le\widehat{\mathcal{R}}_n^+,
    \qquad
    \mathcal{R}_n^-(k)\le\widehat{\mathcal{R}}_n^-,
    \quad \forall k\in\mathbb{K}_n .
\]

\begin{lemma}[Response-Time Bounds]
\label{lem:response_time_bounds}
Under the proposed static scheduling policy, feasible response-time bounds \(\widehat{\mathcal{R}}_n^+\) and \(\widehat{\mathcal{R}}_n^-\) are given by the least fixed points of
\[
    \widehat{\mathcal{R}}_n^+
    =
    \mathcal{B}_n^+
    +
    \sum_{\beta=1}^{N_{\mathcal{B}}}
    \left\lceil
        \frac{\widehat{\mathcal{R}}_n^+}{\mathcal{P}^{\mathcal{B}}_\beta}
    \right\rceil
    \mathcal{C}^{\mathcal{B}}_\beta
    +
    \sum_{\substack{w:\mathcal{T}^{(w)}<\mathcal{T}_n}}
    \left\lceil
        \frac{\widehat{\mathcal{R}}_n^+}{\mathcal{T}^{(w)}}
    \right\rceil
    \mathcal{C}^{\mathbf{Q}}_w ,
\]
and
\[
    \widehat{\mathcal{R}}_n^-
    =
    \mathcal{B}_n^-
    +
    \sum_{\beta=1}^{N_{\mathcal{B}}}
    \left\lceil
        \frac{\widehat{\mathcal{R}}_n^-}{\mathcal{P}^{\mathcal{B}}_\beta}
    \right\rceil
    \mathcal{C}^{\mathcal{B}}_\beta
    +
    \sum_{\substack{w:\mathcal{T}^{(w)}<\mathcal{T}_n}}
    \left\lceil
        \frac{\widehat{\mathcal{R}}_n^-}{\mathcal{T}^{(w)}}
    \right\rceil
    \mathcal{C}^{\mathbf{Q}}_w ,
\]
where
\[
    \mathcal{B}_n^+
    \triangleq
    \mathcal{C}_n^+
    +
    \sum_{\substack{u:\mathcal{T}_u=\mathcal{T}_n\\
            \mathcal{I}(d_u)<\mathcal{I}(d_n)}}
    \mathcal{C}_u^+
\]
and
\[
    \mathcal{B}_n^-
    \triangleq
    \sum_{\substack{u:\mathcal{T}_u=\mathcal{T}_n}}
    \mathcal{C}_u^+
    +
    \sum_{\substack{u:\mathcal{T}_u=\mathcal{T}_n\\
            \mathcal{I}(d_u)>\mathcal{I}(d_n)}}
    \mathcal{C}_u^-
    +
    \mathcal{C}_n^- .
\]
Here, \(\mathcal{C}_u^+\) and \(\mathcal{C}_u^-\) denote the WCETs of \SymbUpdate and \SymbHandle of device \(d_u\), respectively.
\end{lemma}

\proof 
See Appendix~\ref{appendix: response_time_bounds}. 
\endproof

Lemma~\ref{lem:response_time_bounds} gives release-independent completion bounds for \SymbUpdate and \SymbHandle phases. The registry order is already used in this local analysis to tighten the deterministic same-period workload. Specifically, \(\mathcal B_n^+\) only includes the \SymbUpdate prefix before \(d_n\) in the forward traversal, while \(\mathcal B_n^-\) includes all same-period \SymbUpdate phases and only the reverse-\SymbHandle prefix before \(d_n\). Thus, compared with charging all same-period phases uniformly, the ordering information gives tighter bounds \(\widehat{\mathcal R}_n^+\) and \(\widehat{\mathcal R}_n^-\).

The subsequent latency analysis can therefore ignore the exact release-dependent interference and use \(\widehat{\mathcal R}_n^+\) and \(\widehat{\mathcal R}_n^-\) as safe local completion bounds.

We next propagate these bounds through the data-transfer relation along a leaf-to-root path to derive an upper bound on the intra-tree decision latency.


\subsection{Intra-tree Latency Bound}

For clarity, we use \emph{producer} and \emph{consumer} to refer to the data-transfer direction of the current phase.
During \SymbUpdate propagation, the child device is the producer and the parent device is the consumer.
During \SymbHandle propagation, this relation is reversed.

For two adjacent devices on a path, the consumer can observe the producer output only at its own release instants.
At a common release instant, if the proposed scheduling order ensures that the producer output is already available when the consumer reads its input, then the consumer may use the producer instance released at the same tick.
Otherwise, we conservatively require the consumed producer instance to have completed before the consumer release, using the response-time bound in \textbf{Lemma~\ref{lem:response_time_bounds}}.

For the \(i\)-th path \((n_p)_{p=1}^{P_i}\), we define two predicates to distinguish these cases.
For upward \SymbUpdate propagation, let \(\mathcal{U}_{i,u}\), \(u=2,\ldots,P_i\), indicate whether the output of \(\tau_{n_{(u-1)}}^{+}\) is visible to \(\tau_{n_u}^{+}\) at a common release instant:
\[
\begin{aligned}
    \mathcal{U}_{i,u}
    \triangleq {}&
    \left[
        \mathcal{T}_{n_{(u-1)}}
        <
        \mathcal{T}_{n_u}
    \right.
    \\
    &\left.
    {}\lor
        \left(
            \mathcal{T}_{n_{(u-1)}}
            =
            \mathcal{T}_{n_u}
            \land
            \mathcal{I}(d_{n_{(u-1)}})
            <
            \mathcal{I}(d_{n_u})
        \right)
    \right].
\end{aligned}
\]

For downward \SymbHandle propagation, let \(\mathcal{H}_{i,v}\), \(v=1,\ldots,P_i-1\), indicate whether the output of \(\tau_{n_{(P_i-v+1)}}^{-}\) is visible to \(\tau_{n_{(P_i-v)}}^{-}\) at a common release instant:
\[
\begin{aligned}
    \mathcal{H}_{i,v}
    \triangleq {}&
    \left[
        \mathcal{T}_{n_{(P_i-v+1)}}
        <
        \mathcal{T}_{n_{(P_i-v)}}
    \right.
    \\
    &\left.
    {}\lor
        \left(
            \mathcal{T}_{n_{(P_i-v+1)}}
            =
            \mathcal{T}_{n_{(P_i-v)}}
        \right.
    \right.
    \\
    &\left.
    \left.
            {}\land
            \mathcal{I}(d_{n_{(P_i-v+1)}})
            >
            \mathcal{I}(d_{n_{(P_i-v)}})
        \right)
    \right].
\end{aligned}
\]

When the corresponding predicate holds, the consumer release is aligned with the producer release.
Otherwise, it is aligned with the producer completion bound.
The following theorem constructs a release sequence that safely upper-bounds the sequence selected by the data-transfer relation.

\begin{theorem}[Intra-tree Latency Bound]
\label{thm:tree_latency_bound}
    Consider the \(i\)-th leaf-to-root path \((n_p)_{p=1}^{P_i}\in\mathbb{P}_m\) and a final leaf \SymbHandle release index
    \(k_{(2P_i-1)}\in\mathbb{K}_{n_{(2P_i-1)}}\).
    Let \((k_j)_{j=1}^{2P_i-1}\) be recursively defined backward as follows.
    
    For downward \SymbHandle propagation, for \(v=P_i-1,\ldots,1\),
    \[
        k_{(P_i+v-1)}
        =
        \begin{cases}
            \displaystyle
            \left\lfloor
            \frac{k_{(P_i+v)}\delta}{\mathcal{T}_{n_{(P_i-v+1)}}}
            \right\rfloor
            \frac{\mathcal{T}_{n_{(P_i-v+1)}}}{\delta},
            &
            \mathcal{H}_{i,v},
            \\[16pt]
            \displaystyle
            \left\lfloor
            \frac{
                \begin{aligned}
                    &k_{(P_i+v)}\delta\\
                    &{}-\widehat{\mathcal{R}}_{n_{(P_i-v+1)}}^{-}
                \end{aligned}
            }
            {\mathcal{T}_{n_{(P_i-v+1)}}}
            \right\rfloor
            \frac{\mathcal{T}_{n_{(P_i-v+1)}}}{\delta},
            &
            \text{otherwise}.
        \end{cases}
    \]
    
    For upward \SymbUpdate propagation, for \(u=P_i,\ldots,2\),
    \[
        k_{(u-1)}
        =
        \begin{cases}
            \displaystyle
            \left\lfloor
            \frac{k_u\delta}{\mathcal{T}_{n_{(u-1)}}}
            \right\rfloor
            \frac{\mathcal{T}_{n_{(u-1)}}}{\delta},
            &
            \mathcal{U}_{i,u},
            \\[16pt]
            \displaystyle
            \left\lfloor
            \frac{k_u\delta - \widehat{\mathcal{R}}_{n_{(u-1)}}^{+}}
            {\mathcal{T}_{n_{(u-1)}}}
            \right\rfloor
            \frac{\mathcal{T}_{n_{(u-1)}}}{\delta},
            &
            \text{otherwise}.
        \end{cases}
    \]
    
    Then the intra-tree latency of the \(m\)-th tree is upper-bounded by
    \[
        L_m
        \le
        \max_{(n_p)_{p=1}^{P_i}\in\mathbb{P}_m}
        \max_{k_{(2P_i-1)}\in\mathbb{K}_{n_1}}
        \left(
            k_{(2P_i-1)}\delta
            +
            \widehat{\mathcal{R}}_{n_1}^{-}
            -
            k_1\delta
        \right),
    \]
    where \(k_1\) is the first release index obtained by the backward recurrence for the corresponding path and final release.
\end{theorem}

\proof
Inspired by the latency-analysis methodology for data chains of real-time periodic tasks~\cite{kloda2018latency}, we prove Theorem~\ref{thm:tree_latency_bound} in Appendix~\ref{appendix: tree_latency_bound}.
The key insight is that the predicates \(\mathcal{U}_{i,u}\) and \(\mathcal{H}_{i,v}\) identify when adjacent producer--consumer phases have same-tick visibility; otherwise, the proof falls back to the response-time bounds \(\widehat{\mathcal{R}}_n^{+}\) and \(\widehat{\mathcal{R}}_n^{-}\) to conservatively account for release-level waiting.
\endproof


\subsection{Latency under Special Cases}

Since no period relation is imposed on devices within a tree, the general form of Theorem~\ref{thm:tree_latency_bound} is necessarily complex and provides limited intuition.
We therefore derive a simpler bound for a common case in robotic control firmware.

\subsubsection{Harmonic Device Periods}

We first consider tree-structured harmonic periods.
For every dependency edge \((d_{n_1},d_{n_2})\in\mathbb{E}\), where \(d_{n_1}\) is the child and \(d_{n_2}\) is the parent, assume that
\[
    \mathcal{T}_{n_2}
    =
    \lambda_{n_1,n_2}\mathcal{T}_{n_1},
    \qquad
    \lambda_{n_1,n_2}\in\mathbb{N}^{+}.
\]
Thus, the parent period is an integer multiple of the child period.

Under this assumption, upward \SymbUpdate propagation always has same-tick visibility.
If the child has a shorter period, it has higher RMS priority and completes before the parent activation.
If the two devices have the same period, the forward registry order executes the child \SymbUpdate before the parent \SymbUpdate.
Therefore, the corresponding predicate \(\mathcal{U}_{i,u}\) always holds.

For downward \SymbHandle propagation, same-tick visibility holds only when the parent and child have the same period.
If the parent period is larger, the child has higher RMS priority, and the parent \SymbHandle is not guaranteed to complete before the child \SymbHandle at the same tick.
In this case, the consumed parent instance must be selected using the parent \SymbHandle completion bound.

\begin{corollary}[Harmonic Intra-Tree Latency Bound]
\label{cor:harmonic_path_latency_bound}
    Under the harmonic period assumption above, the backward release sequence in Theorem~\ref{thm:tree_latency_bound} can be simplified as follows:
    
    For downward \SymbHandle propagation, for \(v=P_i-1,\ldots,1\),
    \[
        k_{(P_i+v-1)}
        =
        \begin{cases}
            k_{(P_i+v)},\\
            \qquad \qquad \qquad
            \mathcal{T}_{n_{(P_i-v+1)}}
            =
            \mathcal{T}_{n_{(P_i-v)}};
            \\[10pt]
            \displaystyle
            \left\lfloor
            \frac{
                k_{(P_i+v)}\delta
                -
                \widehat{\mathcal{R}}_{n_{(P_i-v+1)}}^{-}
            }
            {\mathcal{T}_{n_{(P_i-v+1)}}}
            \right\rfloor
            \frac{\mathcal{T}_{n_{(P_i-v+1)}}}{\delta},\\
            \qquad \qquad \qquad
            \mathcal{T}_{n_{(P_i-v+1)}}
            >
            \mathcal{T}_{n_{(P_i-v)}} .
        \end{cases}
    \]
    
    For upward \SymbUpdate propagation, for \(u=P_i,\ldots,2\),
    \[
        k_{(u-1)}=k_u .
    \]
    
    Consequently, the intra-tree latency of the \(m\)-th tree satisfies
    \[
        L_m
        \le
        \max_{(n_p)_{p=1}^{P_i}\in\mathbb P_m}
        \max_{k_{(2P_i-1)}\in\mathbb K_{n_1}}
        \left(
            k_{(2P_i-1)}\delta
            +
            \widehat{\mathcal R}_{n_1}^{-}
            -
            k_1\delta
        \right),
    \]
\end{corollary}

\begin{proof}
    Consider any path \((n_p)_{p=1}^{P_i}\in\mathbb P_m\). For upward \SymbUpdate propagation, the harmonic assumption gives
    \[
        \mathcal T_{n_{(u-1)}}\le \mathcal T_{n_u},
        \qquad
        u=2,\ldots,P_i .
    \]
    If the inequality is strict, the child has higher RMS priority; if the periods are equal, the forward registry traversal executes the child before the parent. Hence \(\mathcal U_{i,u}\) holds for every upward edge. Since releases are aligned, the upward recurrence in Theorem~\ref{thm:tree_latency_bound} reduces to
    \[
        k_{(u-1)}=k_u .
    \]

    For downward \SymbHandle propagation, the producer is the parent and the consumer is the child, so the harmonic assumption gives
    \[
        \mathcal T_{n_{(P_i-v+1)}}\ge \mathcal T_{n_{(P_i-v)}} .
    \]
    If the periods are equal, reverse registry traversal executes the parent before the child, so \(\mathcal H_{i,v}\) holds and the recurrence gives
    \[
        k_{(P_i+v-1)}=k_{(P_i+v)} .
    \]
    If the parent period is larger, same-tick visibility is not guaranteed, and Theorem~\ref{thm:tree_latency_bound} uses the completion-bound branch, giving the stated second case.

    Therefore, the recurrence in Theorem~\ref{thm:tree_latency_bound} simplifies exactly as stated. Applying the tree-level bound of Theorem~\ref{thm:tree_latency_bound} with this simplified recurrence proves the corollary.
\end{proof}

\subsubsection{Same-Period Device Trees}

The harmonic bound still permits period-level deferral when adjacent devices have different periods. We now consider the special case where all devices have the same period and aligned releases. In this case, every adjacent producer-consumer pair on any leaf-to-root path is released at the same tick. Since the forward traversal follows the child-to-parent direction for \SymbUpdate and the reverse traversal follows the parent-to-child direction for \SymbHandle, intra-tree propagation introduces no release-level waiting.

\begin{corollary}[Same-Period Intra-Tree Latency]
    \label{cor:same_period_intra_tree_latency}
    Suppose that all devices in \(\mathcal D\) have the same period \(\mathcal T\) and aligned releases. Then, for every tree \((\mathcal D_m,\mathcal E_m)\), the intra-tree latency satisfies
    \[
        L_m
        \le
        \max_{(n_p)_{p=1}^{P_i}\in\mathbb P_m}
        \widehat{\mathcal R}_{n_1}^{-}.
    \]
    Moreover, since there is no shorter-period device group in this case, \(\widehat{\mathcal R}_{n_1}^{-}\) is given by the least fixed point
    \[
        \widehat{\mathcal R}_{n_1}^{-}
        =
        \mathcal B_{n_1}^{-}
        +
        \sum_{\beta=1}^{N_{\mathcal B}}
        \left\lceil
            \frac{\widehat{\mathcal R}_{n_1}^{-}}
            {\mathcal P_\beta^{\mathcal B}}
        \right\rceil
        \mathcal C_\beta^{\mathcal B},
    \]
    where
    \[
        \mathcal B_{n_1}^{-}
        =
        \sum_{d_u\in\mathcal D}
        \mathcal C_u^+
        +
        \sum_{\substack{d_u\in\mathcal D\\
                \mathcal I(d_u)>\mathcal I(d_{n_1})}}
        \mathcal C_u^-
        +
        \mathcal C_{n_1}^- .
    \]
\end{corollary}

\begin{proof}
    Consider any path \((n_p)_{p=1}^{P_i}\in\mathbb P_m\). Since all devices have the same period \(\mathcal T\) and aligned releases, every adjacent producer-consumer pair on the path is released at the same tick. The forward traversal executes each child \SymbUpdate before its parent \SymbUpdate, and the reverse traversal executes each parent \SymbHandle before its child \SymbHandle. Hence all predicates \(\mathcal U_{i,u}\) and \(\mathcal H_{i,v}\) hold. Therefore, Theorem~\ref{thm:tree_latency_bound} gives
    \[
        k_1=k_2=\cdots=k_{(2P_i-1)} .
    \]
    Substituting this equality into Theorem~\ref{thm:tree_latency_bound} yields
    \[
        L_m
        \le
        \max_{(n_p)_{p=1}^{P_i}\in\mathbb P_m}
        \widehat{\mathcal R}_{n_1}^{-}.
    \]

    It remains to simplify \(\widehat{\mathcal R}_{n_1}^{-}\). Since all devices have the same period, the shorter-period device-group term in Lemma~\ref{lem:response_time_bounds} is empty. Thus,
    \[
        \widehat{\mathcal R}_{n_1}^{-}
        =
        \mathcal B_{n_1}^{-}
        +
        \sum_{\beta=1}^{N_{\mathcal B}}
        \left\lceil
        \frac{\widehat{\mathcal R}_{n_1}^{-}}
        {\mathcal P_\beta^{\mathcal B}}
        \right\rceil
        \mathcal C_\beta^{\mathcal B}.
    \]
    Since all devices belong to the same period group, \(\mathcal B_{n_1}^{-}\) takes the stated form. This proves the corollary.
\end{proof}

\begin{remark}
The proofs of Corollary~\ref{cor:harmonic_path_latency_bound} and Corollary~\ref{cor:same_period_intra_tree_latency} show that Theorem~\ref{thm: initialization-order-preservation}, enabled by the framework design, reduces latency at multiple propagation steps.
Since harmonic and same-period device structures are common in robotic systems, the proposed scheduling policy is effective in reducing intra-tree decision latency in practical scenarios.
\end{remark}

%% file: Contents/Experiment.tex
\section{Experiments}

\subsection{Setup}

\subsubsection{Scenario}

Evaluating a control firmware framework is difficult because a fair comparison requires a real project with an existing implementation and development history.
Moreover, changing the baseline to match all functions of the proposed framework may itself affect fairness.
Therefore, we use a case study: we adapt \FineMote to an existing open-source robotic platform and compare firmware-level timing metrics.

We choose the underwater robot \ROV as the evaluation platform.
Its mechanical and electrical design is open-source~\cite{xu2025aucamp}, and its control firmware is also available~\cite{ROV_control}.
The platform uses the \textsc{RoboMaster Type C} board, which is supported by \FineMote.

\ROV has eight thrusters and performs closed-loop attitude and motion control using an IMU and four pressure sensors.
The IMU is connected through UART, while the pressure sensors are connected through an I\textsuperscript{2}C expansion board.
Due to limited peripheral interfaces, the thruster ESCs are driven by an I\textsuperscript{2}C-to-8-channel PWM module.
From the perspective of \FineMote, the system contains 14 device objects: 13 leaf devices and one top-level \ROV object.
These devices form the two-level tree shown in Fig.~\ref{fig:rov_tree}.
The IMU period is \(\SI{5}{\milli\second}\), while all other devices have a period of \(\SI{10}{\milli\second}\).

\begin{figure}[ht]
    \centering
    \includegraphics[width=0.9\linewidth]{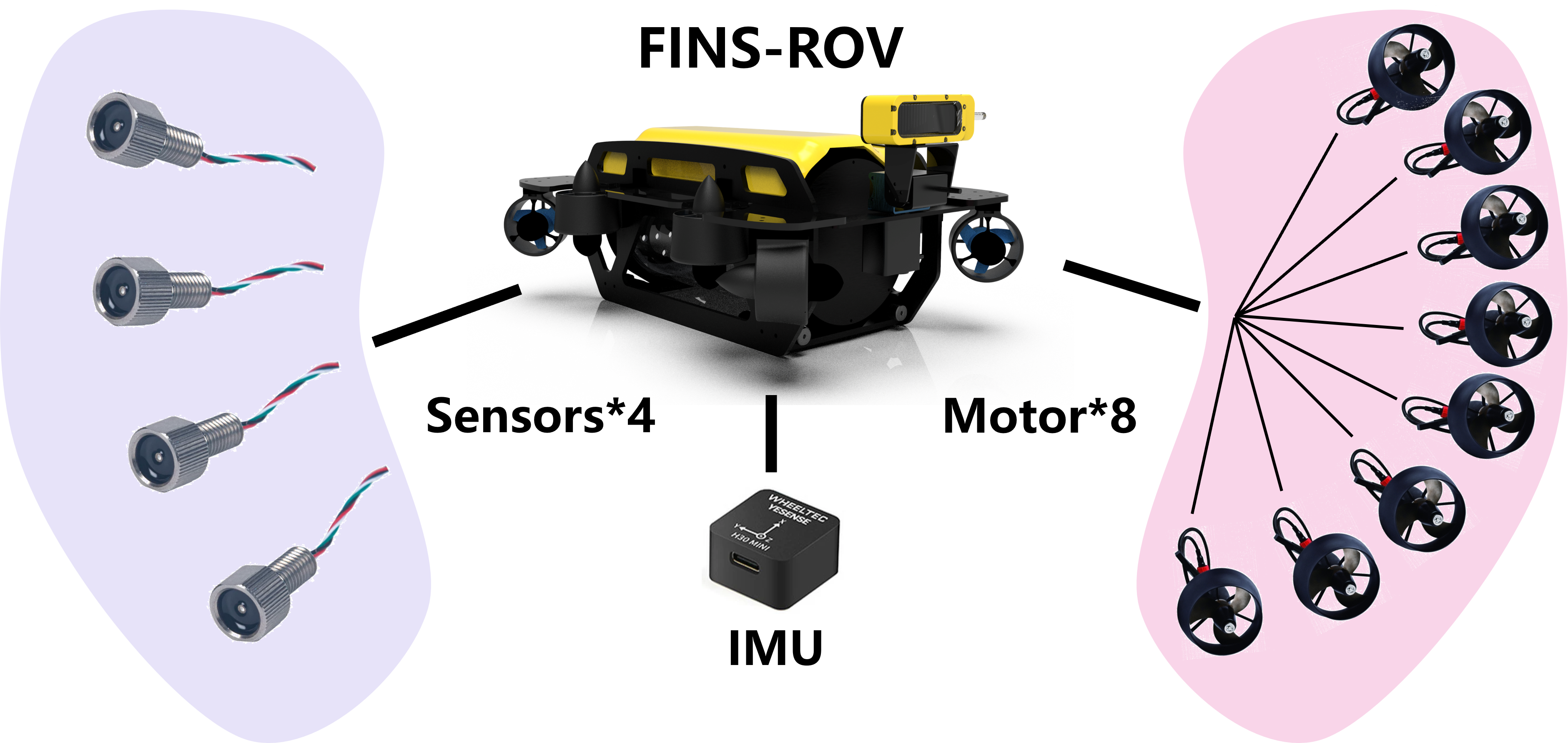}
    \caption{Device tree structure of the \ROV case study.}
    \label{fig:rov_tree}
\end{figure}

\subsubsection{Measurement Method}

To reduce measurement interference, we use a debugger and RTT Viewer to collect timestamps.
RTT transfers data through the debug channel, avoiding extra bus communication in the control path.

We insert lightweight probe functions at the beginning of \SymbUpdate and the end of \SymbHandle.
Based on the sequential execution semantics of C++ statements~\cite{cpp17}, the recorded timestamps preserve the order of measured events.
In our setup, the method provides approximately \(\SI{10}{\micro\second}\)-level precision, which is sufficient for the timing analysis in this experiment.

\subsection{Results}

The baseline firmware uses interrupt-set flags and polling in the main loop, making its latency difficult to define using our formal metric.
We therefore evaluate the firmware using two practical timing metrics.

The first metric is the execution interval between the farthest sensor event and motor output event, which reflects the internal response delay of the firmware.
As shown in Fig.~\ref{fig:latency}, \FineMote achieves an average interval of \(\SI{53.64}{\micro\second}\) with a standard deviation of \(\SI{4.29}{\micro\second}\), while the baseline reaches \(\SI{1610.00}{\micro\second}\) with a standard deviation of \(\SI{10.68}{\micro\second}\).

\begin{figure}[t]
    \centering
    \includegraphics[width=0.9\linewidth]{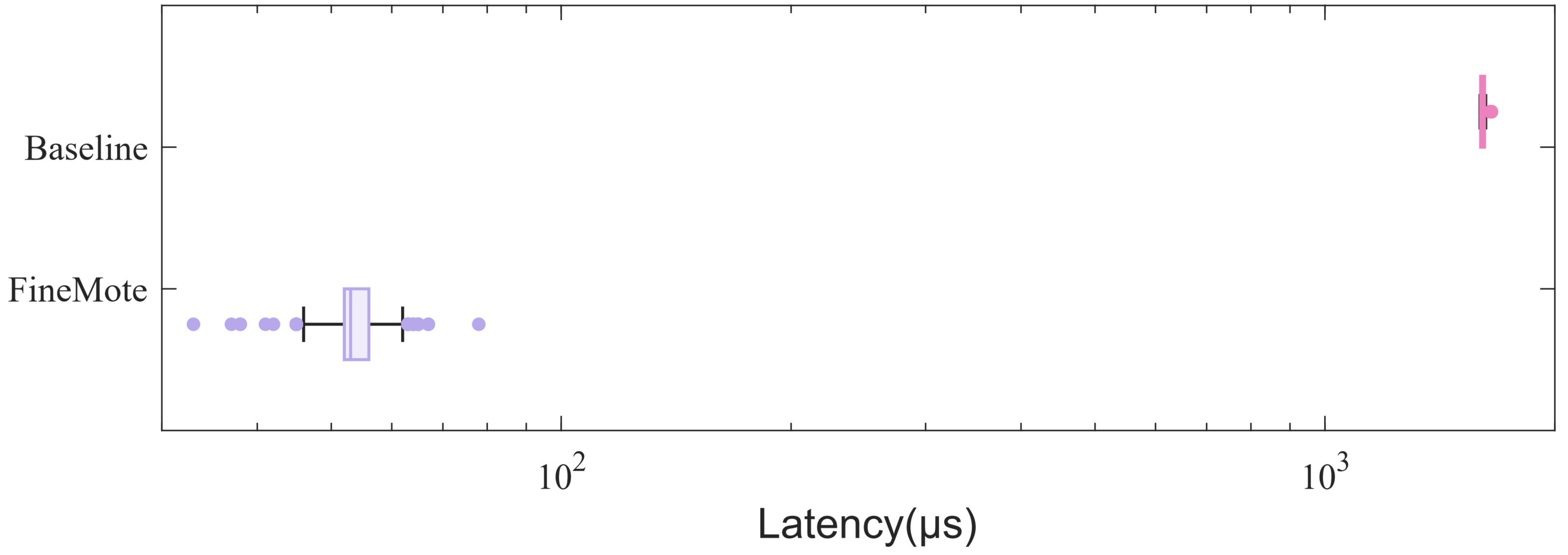}
    \caption{Execution interval between the farthest sensor event and motor output event.}
    \label{fig:latency}
\end{figure}

The second metric is actuator jitter, measured as the standard deviation of the execution interval of the last initialized motor's \SymbHandle.
As shown in Fig.~\ref{fig:jitter}, \FineMote reduces this jitter to \(\SI{5.34}{\micro\second}\), compared with \(\SI{18.87}{\micro\second}\) in the baseline.

\begin{figure}[t]
    \centering
    \includegraphics[width=0.9\linewidth]{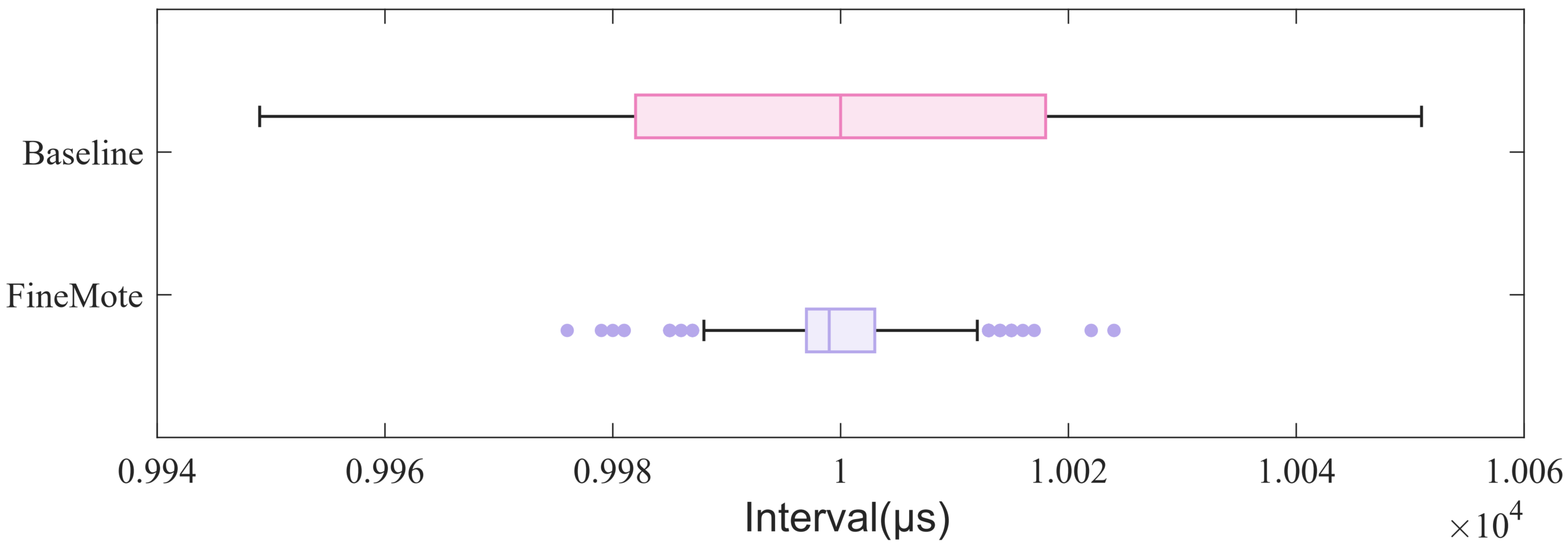}
    \caption{Actuator jitter measured from the last initialized motor.}
    \label{fig:jitter}
\end{figure}

These results show that the proposed framework and scheduling policy provide lower firmware-level response delay and more stable actuator timing in the \ROV case study.

%% file: Contents/Related_Works.tex
\section{Related Works \label{sec: related works}}

Structure-driven generation has been widely explored in robotics and automotive software.
In robotics, ros2\_control~\cite{ros2_control_docs} uses structural information in URDF to describe hardware interfaces and bind them with device-specific implementations.
In the automotive domain, AUTOSAR~\cite{AUTOSAR} adopts an XML-based paradigm to configure and generate control software.
These approaches reduce the engineering effort of firmware or middleware development.
However, they mainly focus on functional integration and code generation, while timing orchestration and scheduling guarantees are not treated as first-class design objectives.

Several studies have addressed real-time scheduling for robotic middleware.
ROS-lite~\cite{azumi2020ros,tajima2024ros} provides a lightweight ROS-compatible framework for embedded many-core platforms.
ROSCH~\cite{saito2018rosch} improves real-time execution in ROS through synchronization support and fixed-priority DAG scheduling.
These systems demonstrate the importance of real-time scheduling in robotic software stacks.
However, they do not target structure-driven firmware generation, nor do they exploit tree-structured device dependencies induced by robot descriptions.

In contrast, this work jointly considers firmware generation and timing orchestration.
The proposed framework exposes schedulable units from tree-structured device objects, statically derives execution order from the construction and registration mechanism, and provides schedulability and latency analysis for the resulting firmware.
Thus, our approach complements existing generation frameworks by making timing behavior an explicit part of the generated control firmware.

%% file: Contents/Conclusion.tex
\section{Conclusion}

This paper presented a control firmware generation framework with static timing orchestration for tree-structured robotic device models.
The proposed framework objectifies heterogeneous low-level device logic, decouples event-triggered communication from periodic control execution, and exposes device-level scheduling units for timing analysis.
By exploiting dependency injection and C++ initialization behavior, the framework derives a deterministic device registration order without explicitly maintaining the complete device forest at runtime.

Based on this registration order, we proposed a static scheduling policy that groups devices by period, assigns Rate-Monotonic Scheduling priorities across period buckets, and executes device phases in an order consistent with tree-structured data propagation.
We formulated the corresponding scheduling problem, established schedulability conditions, and derived intra-tree decision-latency bounds under the proposed policy.
These results show that the framework can satisfy deadline and precedence constraints while reducing latency accumulation within device trees.

We implemented the proposed design and evaluated it on robotic control platforms.
The results demonstrate that static timing orchestration can provide predictable timing behavior with low runtime overhead, while preserving the development benefits of structure-driven firmware generation.
Future work will focus on tighter latency analysis, broader support for heterogeneous communication patterns, and automated extraction of timing parameters from generated firmware.

%% file: Contents/Appendix.tex
\appendix

\subsection{Proof of Theorem~\ref{thm: initialization-order-preservation}}
\label{appendix: initialization-order}

The proof is closely tied to the implementation discipline of the framework.
We first establish the ordering property and then discuss the implementation requirements.

Consider a dependency edge
\[
    (d^{(q)}_{r_1},d^{(q)}_{r_2})\in\mathbb{E}
\]
within a translation unit \(\mathbb{U}_q\), where \(d^{(q)}_{r_1}\) is the child device and \(d^{(q)}_{r_2}\) is the parent device.
Under dependency injection, the parent device is constructed using the child device as a construction argument.
Therefore, by the name lookup rules in \S~6.4 of C++17~\cite{cpp17}, the child device must be declared or defined before it is used in the parent construction.
Hence,
\[
    r_1 < r_2 .
\]

Furthermore, by \S~6.6.3 of C++17~\cite{cpp17}, under the requirements of \emph{partially-ordered initialization}, this source-level order is preserved in the global initialization order.
Thus,
\[
    \mathcal{I}(d^{(q)}_{r_1})
    <
    \mathcal{I}(d^{(q)}_{r_2}) .
\]
This shows that each translation-unit sequence \(\mathbf{U}_q\) is preserved as a subsequence of the global initialization order \(\mathbf{O}\).

It remains to show that all device dependencies are covered.
By \S~6.2 of C++17~\cite{cpp17}, every device object that appears in the firmware must be defined at least once.
Therefore,
\[
    \bigcup_{q=1}^{Q}
    \{\, d \mid d \in \mathbf{U}_q \,\}
    =
    \mathbb{D}.
\]
Consequently, every dependency edge in \(\mathbb{E}\) appears in at least one translation-unit sequence whose order is preserved by partially-ordered initialization.
Combining this with the previous argument yields
\[
    (d_{n_1},d_{n_2})\in\mathbb{E}
    \;\Rightarrow\;
    \mathcal{I}(d_{n_1})
    <
    \mathcal{I}(d_{n_2}) .
\]
Therefore, the global initialization order \(\mathbf{O}\) is consistent with the child-to-parent order of the device forest.
\endproof

The proof uses \emph{partially-ordered initialization}, which can be realized in two ways:
i) All objects in the same device tree can be defined in one translation unit.
ii) If a developer intends to expose a defined device object to other translation units, the object should be declared with the \texttt{inline} specifier.

In either case, \S~6.6.3 ensures that if all primitive devices appear before the composite device in every relevant translation unit, their construction order is preserved.
This property is naturally respected under dependency injection.

\subsection{Proof of Lemma~\ref{lem:response_time_bounds}}
\label{appendix: response_time_bounds}

Consider an arbitrary release index \(k\in\mathbb{K}_n\). We first consider the \SymbUpdate phase \(\tau_n^+\). Under the proposed scheduling policy, all devices with the same period as \(d_n\) are released at the same tick and are executed according to the registry traversal order. Therefore, before \(\tau_n^+\) completes, the deterministic same-period workload consists of \(\tau_n^+\) itself and the \SymbUpdate phases of all same-period devices that appear before \(d_n\) in the registry. This gives the term \(\mathcal{B}_n^+\).

During the response interval of length \(\mathcal{R}_n^+(k)\), each bus response stream \(\beta\) contributes at most
\[
    \left\lceil
        \frac{\mathcal{R}_n^+(k)}
             {\mathcal{P}^{\mathcal{B}}_\beta}
    \right\rceil
    \mathcal{C}^{\mathcal{B}}_\beta
\]
units of interference. Similarly, each higher-priority device period group \(w\) with \(\mathcal{T}^{(w)}<\mathcal{T}_n\) contributes at most
\[
    \left\lceil
        \frac{\mathcal{R}_n^+(k)}
             {\mathcal{T}^{(w)}}
    \right\rceil
    \mathcal{C}^{\mathbf{Q}}_w .
\]
Hence
\[
    \mathcal{R}_n^+(k)
    \le
    \mathcal{B}_n^+
    +
    \sum_{\beta=1}^{N_{\mathcal{B}}}
    \left\lceil
        \frac{\mathcal{R}_n^+(k)}
             {\mathcal{P}^{\mathcal{B}}_\beta}
    \right\rceil
    \mathcal{C}^{\mathcal{B}}_\beta
    +
    \sum_{\substack{w:\mathcal{T}^{(w)}<\mathcal{T}_n}}
    \left\lceil
        \frac{\mathcal{R}_n^+(k)}
             {\mathcal{T}^{(w)}}
    \right\rceil
    \mathcal{C}^{\mathbf{Q}}_w .
\]
By the standard fixed-priority response-time argument, the least fixed point of the corresponding monotone recurrence is a safe release-independent upper bound. Therefore,
\[
    \mathcal{R}_n^+(k)\le\widehat{\mathcal{R}}_n^+ .
\]

The proof for \SymbHandle is identical except for the deterministic same-period workload. Before \(\tau_n^-\) completes, all same-period \SymbUpdate phases must have completed, and the reverse registry traversal executes the \SymbHandle phases of same-period devices appearing after \(d_n\) before executing \(\tau_n^-\). Therefore, the deterministic same-period workload is \(\mathcal{B}_n^-\). The bus-interference and higher-priority device-group interference terms are unchanged. Hence the same fixed-priority response-time argument gives
\[
    \mathcal{R}_n^-(k)\le\widehat{\mathcal{R}}_n^- .
\]

Since \(k\in\mathbb{K}_n\) was arbitrary, both bounds hold for every release of device \(d_n\).
\endproof

\subsection{Proof of Theorem~\ref{thm:tree_latency_bound}}
\label{appendix: tree_latency_bound}

Let \((\kappa_j)_{j=1}^{2P_i-1}\) be the exact release sequence selected by the backward data-transfer relation for the path ending at
\[
    \kappa_{(2P_i-1)}=k_{(2P_i-1)} .
\]
We prove by structural induction that
\[
    \kappa_j \ge k_j,
    \qquad
    j=1,\ldots,2P_i-1 .
\]
The base case holds since \(\kappa_{(2P_i-1)}=k_{(2P_i-1)}\). Assume that the claim holds for the current consumer release in a backward propagation step.

Consider a downward \SymbHandle propagation step from \(d_{n_{(P_i-v+1)}}\) to \(d_{n_{(P_i-v)}}\), where \(v=P_i-1,\ldots,1\). By the backward data-transfer relation,
\[
\begin{aligned}
    \kappa_{(P_i+v-1)}
    &=
    \max
    \Bigl\{
        q\in\mathbb K_{n_{(P_i-v+1)}}
        {}\mid
    \\
    &\quad
        F(\tau_{n_{(P_i-v+1)}}^-,q)
        \le
        S(\tau_{n_{(P_i-v)}}^-,\kappa_{(P_i+v)})
    \Bigr\}.
\end{aligned}
\]

If \(\mathcal H_{i,v}\) holds, the scheduler guarantees that the producer output is visible to the consumer for the latest producer release no later than the consumer release. Hence
\[
    \kappa_{(P_i+v-1)}
    \ge
    \left\lfloor
    \frac{\kappa_{(P_i+v)}\delta}{\mathcal T_{n_{(P_i-v+1)}}}
    \right\rfloor
    \frac{\mathcal T_{n_{(P_i-v+1)}}}{\delta}.
\]
By the induction hypothesis \(\kappa_{(P_i+v)}\ge k_{(P_i+v)}\) and the monotonicity of the floor alignment,
\[
\begin{aligned}
    &\left\lfloor
    \frac{\kappa_{(P_i+v)}\delta}
         {\mathcal T_{n_{(P_i-v+1)}}}
    \right\rfloor
    \frac{\mathcal T_{n_{(P_i-v+1)}}}{\delta}
    \\
    &\qquad\ge
    \left\lfloor
    \frac{k_{(P_i+v)}\delta}
         {\mathcal T_{n_{(P_i-v+1)}}}
    \right\rfloor
    \frac{\mathcal T_{n_{(P_i-v+1)}}}{\delta}
    =
    k_{(P_i+v-1)} .
\end{aligned}
\]

Thus \(\kappa_{(P_i+v-1)}\ge k_{(P_i+v-1)}\).

If \(\mathcal H_{i,v}\) does not hold, visibility from the latest producer release no later than the consumer release is not guaranteed. Let
\[
    q
    =
    \left\lfloor
    \frac{\kappa_{(P_i+v)}\delta-\widehat{\mathcal R}_{n_{(P_i-v+1)}}^{-}}
    {\mathcal T_{n_{(P_i-v+1)}}}
    \right\rfloor
    \frac{\mathcal T_{n_{(P_i-v+1)}}}{\delta}.
\]
By Lemma~\ref{lem:response_time_bounds},
\[
    F(\tau_{n_{(P_i-v+1)}}^-,q)
    \le
    q\delta
    +
    \widehat{\mathcal R}_{n_{(P_i-v+1)}}^{-}
    \le
    \kappa_{(P_i+v)}\delta .
\]
Since a task cannot start before its release instant,
\[
    \kappa_{(P_i+v)}\delta
    \le
    S(\tau_{n_{(P_i-v)}}^-,\kappa_{(P_i+v)}) .
\]
Therefore,
\[
    F(\tau_{n_{(P_i-v+1)}}^-,q)
    \le
    S(\tau_{n_{(P_i-v)}}^-,\kappa_{(P_i+v)}) .
\]
Hence \(q\) is an observable producer release. Since \(\kappa_{(P_i+v-1)}\) is the latest observable release selected by the backward data-transfer relation, we obtain
\[
    \kappa_{(P_i+v-1)}\ge q .
\]
By the induction hypothesis and the monotonicity of the floor alignment,
\[
    q
    \ge
    \left\lfloor
    \frac{k_{(P_i+v)}\delta-\widehat{\mathcal R}_{n_{(P_i-v+1)}}^{-}}
    {\mathcal T_{n_{(P_i-v+1)}}}
    \right\rfloor
    \frac{\mathcal T_{n_{(P_i-v+1)}}}{\delta}
    =
    k_{(P_i+v-1)} .
\]
Thus \(\kappa_{(P_i+v-1)}\ge k_{(P_i+v-1)}\).

The upward \SymbUpdate propagation is analogous. For the transfer from \(d_{n_{(u-1)}}\) to \(d_{n_u}\), where \(u=P_i,\ldots,2\), the same argument applies with \(\mathcal U_{i,u}\), \(\widehat{\mathcal R}_{n_{(u-1)}}^+\), \(\tau^+\), \(n_{(u-1)}\), and \(n_u\) replacing \(\mathcal H_{i,v}\), \(\widehat{\mathcal R}_{n_{(P_i-v+1)}}^-\), \(\tau^-\), \(n_{(P_i-v+1)}\), and \(n_{(P_i-v)}\), respectively. Thus, for every upward step,
\[
    \kappa_{(u-1)}\ge k_{(u-1)} .
\]

Therefore, by backward induction over the downward and upward propagation steps,
\[
    \kappa_1\ge k_1 .
\]

The exact decision latency of this path under the final release \(k_{(2P_i-1)}\) is
\[
    F(\tau_{n_1}^-,\kappa_{(2P_i-1)})
    -
    S(\tau_{n_1}^+,\kappa_1).
\]
Using Lemma~\ref{lem:response_time_bounds} and \(\kappa_{(2P_i-1)}=k_{(2P_i-1)}\), we have
\[
    F(\tau_{n_1}^-,\kappa_{(2P_i-1)})
    -
    S(\tau_{n_1}^+,\kappa_1)
    \le
    k_{(2P_i-1)}\delta
    +
    \widehat{\mathcal R}_{n_1}^{-}
    -
    S(\tau_{n_1}^+,\kappa_1).
\]
Since a task cannot start before its release instant and \(\kappa_1\ge k_1\),
\[
    S(\tau_{n_1}^+,\kappa_1)
    \ge
    \kappa_1\delta
    \ge
    k_1\delta .
\]
Therefore,
\[
    F(\tau_{n_1}^-,\kappa_{(2P_i-1)})
    -
    S(\tau_{n_1}^+,\kappa_1)
    \le
    k_{(2P_i-1)}\delta
    +
    \widehat{\mathcal R}_{n_1}^{-}
    -
    k_1\delta .
\]
Thus, the intra-tree latency is bounded by
\[
    L_m
    \le
    \max_{(n_p)_{p=1}^{P_i}\in\mathbb{P}_m}
    \max_{k_{(2P_i-1)}\in\mathbb{K}_{n_1}}
    \left(
        k_{(2P_i-1)}\delta
        +
        \widehat{\mathcal{R}}_{n_1}^{-}
        -
        k_1\delta
    \right),
\]
\endproof

%% file: Refs.bib
@inproceedings{azumi2020ros,
  title={ROS-lite: ROS framework for NoC-based embedded many-core platform},
  author={Azumi, Takuya and Maruyama, Yuya and Kato, Shinpei},
  booktitle={IEEE/RSJ International Conference on Intelligent Robots and Systems (IROS)},
  year={2020},
  organization={IEEE}
}

@inproceedings{tajima2024ros,
  title={ROS-lite2: Autonomous-driving Software Platform for Clustered Many-core Processor},
  author={Tajima, Yuta and Tsunoda, Shuhei and Azumi, Takuya},
  booktitle={IEEE/RSJ International Conference on Intelligent Robots and Systems (IROS)},
  year={2024},
  organization={IEEE}
}

@inproceedings{saito2018rosch,
  title={Rosch: real-time scheduling framework for ros},
  author={Saito, Yukihiro and Sato, Futoshi and Azumi, Takuya and Kato, Shinpei and Nishio, Nobuhiko},
  booktitle={International Conference on Embedded and Real-Time Computing Systems and Applications (RTCSA)},
  year={2018},
  organization={IEEE}
}

@misc{ros2_control_docs,
  author = {{Open Robotics}},
  title  = {ros2\_control: Robot Control Framework for {ROS 2}},
  year   = {2021},
  howpublished = {\url{https://control.ros.org/}},
  note   = {Accessed: 2025-08-03}
}

@book{gamma1995design,
  title={Design patterns: elements of reusable object-oriented software},
  author={Gamma, Erich and Helm, Richard and Johnson, Ralph and Vlissides, John},
  year={1995},
  publisher={Pearson Deutschland GmbH}
}

@book{mcconnell2004code,
  title={Code complete},
  author={McConnell, Steve},
  year={2004},
  publisher={Pearson Education}
}

@article{liu1973scheduling,
  title={Scheduling algorithms for multiprogramming in a hard-real-time environment},
  author={Liu, Chung Laung and Layland, James W},
  journal={Journal of the ACM},
  volume={20},
  number={1},
  pages={46--61},
  year={1973},
  publisher={ACM}
}

@article{korsah2013comprehensive,
  title={A comprehensive taxonomy for multi-robot task allocation},
  author={Korsah, G Ayorkor and Stentz, Anthony and Dias, M Bernardine},
  journal={The International Journal of Robotics Research},
  volume={32},
  number={12},
  pages={1495--1512},
  year={2013},
  publisher={SAGE Publications}
}

@inproceedings{kloda2018latency,
  title={Latency analysis for data chains of real-time periodic tasks},
  author={Kloda, Tomasz and Bertout, Antoine and Sorel, Yves},
  booktitle={International Conference on Emerging Technologies and Factory Automation (ETFA)},
  year={2018},
  organization={IEEE}
}

@standard{cpp17,
  title        = {Programming Languages — C++},
  institution  = {International Organization for Standardization},
  year         = {2017},
  type         = {ISO/IEC},
  number       = {14882:2017},
}

@misc{SIOF,
  author       = {{cppreference.com contributors}},
  title        = {{Static Initialization Order Fiasco}},
  howpublished = {\url{https://en.cppreference.com/cpp/language/siof}},
  note         = {Accessed: 2026-05-26}
}

@misc{URDF,
  author       = {{ROS Wiki}},
  title        = {{XML Robot Description Format (URDF)}},
  howpublished = {\url{https://wiki.ros.org/urdf}},
  note         = {Accessed: 2026-05-26}
}

@misc{xacro,
  author       = {{ROS Wiki}},
  title        = {{xacro}},
  howpublished = {\url{https://wiki.ros.org/xacro}},
  note         = {Accessed: 2026-05-26}
}

@misc{AUTOSAR,
  author       = {{AUTOSAR}},
  title        = {{AUTOSAR}},
  howpublished = {\url{https://www.autosar.org}},
  note         = {Accessed: 2026-05-26}
}

@misc{ROV,
  author = {{IWIN-FINS Lab}},
  title  = {FinsROV: an Underwater Camera-Based Multi-robot platform},
  year   = {2022},
  howpublished = {\url{https://github.com/FPJ-GAOGE/FinsROV-An-Underwater-Camera-Based-Multi-Robot-Platform}},
  note   = {Accessed: 2026-05-26}
}

@article{xu2025aucamp,
  title={Aucamp: An Underwater Camera-Based Multi-Robot Platform with Low-Cost, Distributed, and Robust Localization},
  author={Xu, Jisheng and Lin, Ding and Fong, Pangkit and Fang, Chongrong and Duan, Xiaoming and He, Jianping},
  journal={arXiv preprint arXiv:2506.09876},
  year={2025}
}

@misc{ROV_control,
  author = {{IWIN-FINS Lab}},
  title  = {FinsROV: an Underwater Camera-Based Multi-robot platform},
  year   = {2022},
  howpublished = {\url{https://github.com/FPJ-GAOGE/FinsROV-An-Underwater-Camera-Based-Multi-Robot-Platform}},
  note   = {Accessed: 2026-05-26}
}
